\documentclass[12pt]{article}
\usepackage[utf8]{inputenc}
\usepackage[sectionbib]{natbib}

\title{Compact Autoregressive Network}
\author{Di Wang$^\dagger$, Feiqing Huang$^\dagger$, Jingyu Zhao$^\dagger$, Guodong Li$^\dagger$, Guangjian Tian$^\ddagger$\\
\textit{$\dagger$ Department of Statistics and Actuarial Science, University of Hong Kong}\\
\textit{$\ddagger$ Huawei Noah's Ark Lab, Hong Kong, China}}

\usepackage[margin=1in]{geometry}
\usepackage{graphicx}
\usepackage{mathrsfs}
\usepackage{multirow}
\usepackage{amsfonts}
\usepackage{amsmath}
\usepackage{comment}
\usepackage{graphicx}
\usepackage{amsthm}
\usepackage{color}
\usepackage{mathtools}
\usepackage{algorithm}
\usepackage{sectsty}
\usepackage{makecell}
\usepackage{amsmath}
\usepackage{amssymb}
\usepackage[toc,page]{appendix}
\usepackage[mathscr]{euscript}
\usepackage[toc]{appendix}

\let\counterwithin\relax
\usepackage{chngcntr}
\usepackage{apptools}
\usepackage{booktabs}
\usepackage{caption}
\usepackage{lscape}
\usepackage{tabularx} %
\usepackage[utf8]{inputenc} %
\usepackage[sectionbib]{natbib} %
\usepackage{array} %
\usepackage{arydshln}
\usepackage{soul}
\usepackage{tikz}
\AtAppendix{\counterwithin{lemma}{section}}

\newtheorem{condition}{Condition}
\newtheorem{definition}{Definition}
\newtheorem{lemma}{Lemma}
\newtheorem{remark}{Remark}

\newtheorem{theorem}{Theorem}

\newcolumntype{P}[1]{>{\centering\arraybackslash}p{#1}} %

\newcommand{\norm}[1]{\left\lVert#1\right\rVert}

\usepackage[mathscr]{euscript}
\usepackage{mathtools}
\usepackage{amsthm}
\usepackage{amsfonts}
\usepackage{amssymb}
\usepackage{subfig}
\usepackage{color}
\usepackage{multirow}
\usepackage{makecell}
\usepackage{booktabs}
\usepackage{tikz}
\newcommand{\bm}{\boldsymbol}
\newcommand{\cm}[1]{\mbox{\boldmath$\mathscr{#1}$}}

\mathtoolsset{showonlyrefs}

\newcommand{\greenline}{\raisebox{2pt}{\tikz{\draw[-,black!40!green,solid,line width = 1pt](0,0) -- (5mm,0);}}}
\newcommand{\blueline}{\raisebox{2pt}{\tikz{\draw[-,blue, dash pattern=on 6pt off 2pt on 6pt off 2pt,line width = 1pt](0,0) -- (5mm,0);}}}
\newcommand{\blackline}{\raisebox{2pt}{\tikz{\draw[-,black, dash pattern=on 1pt off 1.5pt on 8pt off 1.5pt, line width = 1pt](0,0) -- (5mm,0);}}}

\newcommand{\greenlinereal}{\raisebox{2pt}{\tikz{\draw[-, black!40!green,dash pattern=on 6pt off 2pt on 6pt off 2pt,line width = 1pt, line width = 1pt](0,0) -- (5mm,0);}}}
\newcommand{\bluelinereal}{\raisebox{2pt}{\tikz{\draw[-,blue, dash pattern=on 1pt off 1.5pt on 8pt off 1.5pt, line width = 1pt](0,0) -- (5mm,0);}}}
\newcommand{\blacklinereal}{\raisebox{2pt}{\tikz{\draw[-,black,solid,line width = 1pt](0,0) -- (5mm,0);}}}

\mathtoolsset{showonlyrefs}

\begin{document}

\setlength{\parindent}{16pt}

\maketitle

\begin{abstract}
	Autoregressive networks can achieve promising performance in many sequence modeling tasks with short-range dependence. However, when handling high-dimensional inputs and outputs, the huge amount of parameters in the network lead to expensive computational cost and low learning efficiency. The problem can be alleviated slightly by introducing one more narrow hidden layer to the network, but the sample size required to achieve a certain training error is still large. To address this challenge, we rearrange the weight matrices of a linear autoregressive network into a tensor form, and then make use of Tucker decomposition to represent low-rank structures. This leads to a novel compact autoregressive network, called Tucker AutoRegressive (TAR) net. Interestingly, the TAR net can be applied to sequences with long-range dependence since the dimension along the sequential order is reduced. Theoretical studies show that the TAR net improves the learning efficiency, and requires much fewer samples for model training. Experiments on synthetic and real-world datasets demonstrate the promising performance of the proposed compact network.
\end{abstract}


\section{Introduction}
Sequence modeling has been used to address a broad range of applications including macroeconomic time series forecasting, financial asset management, speech recognition and machine translation. Recurrent neural networks (RNN) and their variants, such as Long-Short Term Memory \citep{hochreiter1997long} and Gated Recurrent Unit \citep{cho2014learning}, are commonly used as the default architecture or even the synonym of sequence modeling by deep learning practitioners \citep{goodfellow2016deep}.
In the meanwhile, especially for high-dimensional time series, we may also consider the autoregressive modeling or multi-task learning,
\begin{equation}
\bm{\widehat{y}}_{t}=f(\bm{y}_{t-1},\bm{y}_{t-2},\dots,\bm{y}_{t-P}),
\label{eq:multitask}
\end{equation}
where the output $\bm{\widehat{y}}_t$ and each input $\bm{y}_{t-i}$ are $N$-dimensional, and the lag $P$ can be very large for accomodating sequential dependence.
Some non-recurrent feed-forward networks with convolutional or other certain architectures have been proposed recently for sequence modeling, and are shown to have state-of-the-art accuracy. For example, some autoregressive networks, such as PixelCNN \citep{van2016conditional} and WaveNet \citep{oord2016wavenet} for image and audio sequence modeling, are compelling alternatives to the recurrent networks.

This paper aims at the autoregressive model \eqref{eq:multitask} with a large number of sequences.
This problem can be implemented by a fully connected network with $NP$ inputs and $N$ outputs. The number of weights will be very large when the number of sequences $N$ is large, and it will be much larger if the data have long-range sequential dependence.
This will lead to excessive computational burden and low learning efficiency.
Recently, \citet{du2018many} showed that the sample complexity in training a convolutional neural network (CNN) is directly related to network complexity, which indicates that compact models are highly desirable when available samples have limited sizes.

To reduce the redundancy of parameters in neural networks, many low-rank based approaches have been investigated.
One is to reparametrize the model, and then to modify the network architecture accordingly. Modification of architectures for model compression can be found from the early history of neural networks \citep{fontaine1997nonlinear,grezl2007probabilistic}.
For example, a bottleneck layer with a smaller number of units can be imposed to constrain the amount of information traversing the network, and to force a compact representation of the original inputs in a multilayer perceptron (MLP) or an autoencoder \citep{hinton2006reducing}.
The bottleneck architecture is equivalent to a fully connected network with a low-rank constraint on the weight matrix in a linear network.

Another approach is to directly constrain the rank of parameter matrices. For instance, \citet{denil2013predicting} demonstrated significant redundancy in large CNNs, and proposed a low-rank structure of weight matrices to reduce it.
If we treat weights in a layer as a multi-dimensional tensor, tensor decomposition methods can then be employed to represent the low-rank structure, and hence compress the network. Among these works, \citet{lebedev2014speeding} applied the CP decomposition for the 4D kernel of a single convolution layer to speed up CNN, and \citet{jaderberg2014speeding} proposed to construct a low-rank basis of filters to exploit cross-channel or filter redundancy. \citet{kim2015compression} utilized the Tucker decomposition to compress the whole network by decomposing convolution and fully connected layers. The tensor train format was employed in \citet{novikov2015tensorizing} to reduce the parameters in fully connected layers. Several tensor decomposition methods were also applied to compress RNNs \citep{tjandra2018tensor,ye2018learning,pan2019compressing}. In spite of the empirical success of low-rank matrix and tensor approaches in the literature, theoretical studies for learning efficiency are still limited.

A fully connected autoregressive network for \eqref{eq:multitask} will have $N^2P$ weights, and it will reduce to $Nr+NPr$ for an MLP with one hidden layer and $r$ hidden units.
The bottleneck architecture still has too many parameters and, more importantly, it does not attempt to explore the possible compact structure along the sequential order. 
We first simplify the autoregressive network into a touchable framework, by rearranging all weights into a tensor. We further apply Tucker decomposition to introduce a low-dimensional structure and translate it into a compact autoregressive network, 
called Tucker AutoRegressive (TAR) net. It is a special compact CNN with interpretable architecture. 
Different from the original autoregressive network, the TAR net is more suitable for sequences with long-range dependence since the dimension along the sequential order is reduced.

There are three main contributions in this paper:

1. We innovatively tensorize weight matrices to create an extra dimension to account for the sequential order and apply tensor decomposition to exploit the low-dimensional structure along all directions. Therefore, the resulting network can handle sequences with long-range dependence.

2. We provide theoretical guidance on the sample complexity of the proposed network.
Our problem is more challenging than other supervised learning problems owing to the strong dependency in sequential samples and the multi-task learning nature.
Moreover, our sample complexity analysis can be extended to other feed-forward networks.

3. The proposed compact autoregressive network can flexibly accommodate nonlinear mappings, and offer physical interpretations by extracting explainable latent features.

The rest of the paper is organized as follows. Section 2 proposes the linear autoregressive networks with low-rank structures and presents a sample complexity analysis for the low-rank networks. Section 3 introduces the Tucker autoregressive net by reformulating the single-layer network with low-rank structure to a compact multi-layer CNN form. Extensive experiments on synthetic and real datasets are presented in Section 4. Proofs of theorems and detailed information for the real dataset are provided in the Appendix.

\section{Linear Autoregressive Network}

This section demonstrates the methodology by considering a linear version of \eqref{eq:multitask}, 
and theoretically studies the sample complexity of the corresponding network.

\subsection{Preliminaries and Background}

\subsubsection{Notation} We follow the notations in \citet{kolda2009tensor} to denote vectors by lowercase boldface letters, e.g. $\bm{a}$; matrices by capital boldface letters, e.g. $\bm{A}$; tensors of order 3 or higher by Euler script boldface letters, e.g. $\cm{A}$. For a generic $d$\textsuperscript{th}-order tensor $\cm{A}\in\mathbb{R}^{p_1\times\cdots\times p_d}$, denote its elements by $\cm{A}(i_1,i_2,\dots,i_d)$ and unfolding of $\cm{A}$ along the $n$-mode by $\cm{A}_{(n)}$, where the columns of $\cm{A}_{(n)}$ are the $n$-mode vectors of $\cm{A}$, for $n=1,2,\dots,d$. The vectorization operation is denoted by $\text{vec}(\cdot)$. The inner product of two tensors $\cm{A},\cm{B}\in\mathbb{R}^{p_1\times\cdots\times p_d}$ is defined as $\langle\cm{A},\cm{B}\rangle=\sum_{i_1}\cdots\sum_{i_d}\cm{A}(i_1,\dots,i_d)\cm{B}(i_1,\dots,i_d)$. The Frobenius norm of a tensor $\cm{A}$ is defined as $\|\cm{A}\|_{\text{F}}=\sqrt{\langle\cm{A},\cm{A}\rangle}$. The mode-$n$ multiplication $\times_n$ of a tensor $\cm{A}\in\mathbb{R}^{p_1\times\cdots\times p_d}$ and a matrix $\bm{B}\in\mathbb{R}^{q_n\times p_n}$ is defined as
\begin{equation*}
(\cm{A}\times_n\bm{B})(i_1,\dots,j_n,\dots,i_d)
=\sum_{i_n=1}^{p_n}\cm{A}(i_1,\dots,i_n,\dots,i_d)\bm{B}(j_n,i_n),
\end{equation*}
for $n=1,\dots,d$, respectively.
For a generic symmetric matrix $\bm{A}$, $\lambda_{\max}(\bm{A})$ and $\lambda_{\min}(\bm{A})$ represent its largest and smallest eigenvalues, respectively.

\subsubsection{Tucker decomposition} The Tucker ranks of $\cm{A}$ are defined as the matrix ranks of the unfoldings of $\cm{A}$ along all modes, namely $\text{rank}_i(\cm{A})=\text{rank}(\cm{A}_{(i)})$, $i=1,\dots,d$. If the Tucker ranks of $\cm{A}$ are $r_1,\dots,r_d$, where $1\leq r_i\leq p_i$, there exist a tensor $\cm{G}\in\mathbb{R}^{r_1\times\cdots\times r_d}$ and matrices $\bm{U}_i\in\mathbb{R}^{p_i\times r_i}$, such that
\begin{equation}
\label{eq:tucker}
\cm{A}=\cm{G}\times_1\bm{U}_1\times_2\bm{U}_2\cdots\times_d\bm{U}_d,
\end{equation}
which is known as Tucker decomposition \citep{tucker1966some}, and denoted by $\cm{A}=[\![\cm{G};\bm{U}_1,\bm{U}_2,\dots,\bm{U}_d]\!]$. With the Tucker decomposition \eqref{eq:tucker}, the $n$-mode matricization of $\cm{A}$ can be written as
\begin{equation}
\cm{A}_{(n)}=\bm{U}_n\cm{G}_{(n)}(\bm{U}_d\otimes\cdots\otimes\bm{U}_{n+1}\otimes\bm{U}_{n-1}\otimes\cdots\otimes\bm{U}_1)^\top,
\end{equation}
where $\otimes$ denotes the Kronecker product for matrices.

\subsection{Linear Autoregressive Network}

Consider a linear autoregressive network,
\[
\bm{h}_t=\bm{A}_1\bm{y}_{t-1}+\bm{A}_2\bm{y}_{t-2}+\cdots+\bm{A}_{t-P}\bm{y}_{t-P}+\bm{b},
\]
where $\bm{h}_t=\bm{\widehat{y}}_t$ is the output, $\bm{A}_i$s are $N\times N$ weight matrices, and $\bm{b}$ is the bias vector.
Let $\bm{x}_t=(\bm{y}_{t-1}^\top,\dots,\bm{y}_{t-P}^\top)^\top$ be the $NP$-dimensional inputs. We can rewrite it into a fully connected network,
\begin{equation}\label{eq:VAR}
\bm{h}_t=\bm{W}\bm{x}_t+\bm{b},
\end{equation}
for $t=1,\dots,T$, where $\bm{W}=(\bm{A}_1,...,\bm{A}_P)\in \mathbb{R}^{N\times NP}$ is the weight matrix.
Note that $T$ denotes the effective sample size, which is the number of samples for training. In other words, the total length of the sequential data is $T+P$.

To reduce the dimension of $\bm{W}$, a common strategy is to constrain the rank of $\bm{W}$ to be $r$, which is much smaller than $N$. The low-rank weight matrix $\bm{W}$ can be factorized as $\bm{W}=\bm{A}\bm{B}$, where $\bm{A}$ is a $N\times r$ matrix and $\bm{B}$ is a $r\times NP$ matrix, and the fully connected network can be transformed into
\begin{equation}
\label{eq:lowrank}
\bm{h}_t=\bm{A}\bm{B}\bm{x}_t+\bm{b}.
\end{equation}
The matrix factorization reduces the number of parameters in $\bm{W}$ from $N^2P$ to $Nr+NPr$. However, if both $N$ and $P$ are large, the weight matrix $\bm{B}$ is still of large size.

\begin{figure}[t]
	\centering
	\includegraphics[width=\columnwidth]{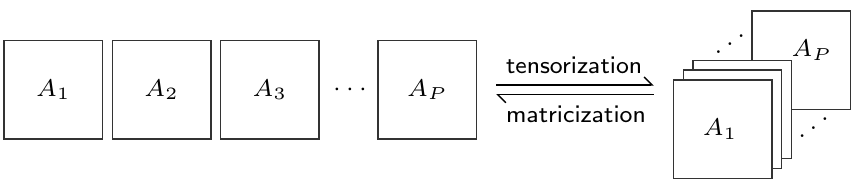}
	\label{fg:tensor}
	\caption{\label{fig1}Rearranging $P$ weight matrices of a linear autoregressive network into a tensor.}
\end{figure}

We alternatively rearrange the weight matrices $\bm{A}_i$s into a 3\textsuperscript{rd}-order tensor $\cm{W}\in\mathbb{R}^{N\times N\times P}$ such that $\cm{W}_{(1)}=\bm{W}$; see Figure \ref{fig1} for the illustration.
The Tucker decomposition can then be applied to reduce the dimension from three modes simultaneously. If the low-Tucker-rank structure is applied on $\cm{W}$ with ranks $r_1,r_2,r_3$, the network becomes
\begin{equation}
\label{eq:lowrankVAR}
\bm{h}_t=\bm{U}_1\cm{G}_{(1)}(\bm{U}_3\otimes\bm{U}_2)^\top\bm{x}_t+\bm{b},
\end{equation}
by Tucker decomposition $\cm{W}=[\![\cm{G};\bm{U}_1,\bm{U}_2,\bm{U}_3]\!]$. The Tucker decomposition further reduces the dimension from the other two modes of low-rank structure in \eqref{eq:lowrank}, while the low-rankness of $\bm{W}$ only considers the low-dimensional structure on the 1-mode of $\cm{W}$ but ignores the possible compact structure on the other two modes.

We train the network based on the squared loss. For simplicity, each sequence is subtracted by its mean, so the bias vector $\bm{b}$ can be disregarded. The weight matrix or tensor in \eqref{eq:VAR}, \eqref{eq:lowrank} and \eqref{eq:lowrankVAR} can be trained, respectively, by minimizing the following ordinary least squares (OLS), low-rank (LR) and low-Tucker-rank (LTR) objective functions,
\begin{equation*}\begin{split}
&\bm{\widehat{W}}_{\textup{OLS}}=\underset{\bm{W}}{\arg\min}\frac{1}{T}\sum_{t=1}^T\|\bm{y}_t-\bm{W}\bm{x}_t\|_2^2,\\
&\bm{\widehat{W}}_{\textup{LR}}=\bm{\widehat{A}}\bm{\widehat{B}}=\underset{\bm{A},\bm{B}}{\arg\min}\frac{1}{T}\sum_{t=1}^T\|\bm{y}_t-\bm{A}\bm{B}\bm{x}_t\|_2^2,\\
&\cm{\widehat{W}}_{\textup{LTR}}=[\![\cm{\widehat{G}};\bm{\widehat{U}}_1,\bm{\widehat{U}}_2,\bm{\widehat{U}}_3]\!]=\underset{{\scriptsize\cm{G}},\bm{U}_1,\bm{U}_2,\bm{U}_3}{\arg\min}~\frac{1}{T}\sum_{t=1}^T\|\bm{y}_t-\bm{U}_1\cm{G}_{(1)}(\bm{U}_3\otimes\bm{U}_2)^\top\bm{x}_t\|_2^2.
\end{split}
\end{equation*}
These three minimizers are called OLS, LR and LTR estimators of weights in the linear autoregressive network, respectively.

The matrix factorization or tensor Tucker decomposition is not unique. Conventionally, orthogonal constraints can be applied to these components to address the uniqueness issue. However, we do not impose any constraints on the components to simplify the optimization and mainly focus on the whole weight matrix or tensor instead of its decomposition.

\subsection{Sample Complexity Analysis}
The sample complexity of a neural network is defined as the training sample size requirement to obtain a certain training error with a high probability, and is a reasonable measure of learning efficiency.
We conduct a sample complexity analysis for the three estimators, $\bm{\widehat{W}}_{\textup{OLS}}$, $\bm{\widehat{W}}_{\textup{LR}}$ and $\cm{\widehat{W}}_{\textup{LTR}}$, under the high-dimensional setting by allowing both $N$ and $P$ to grow with the sample size $T$ at arbitrary rates.

We further assume that the sequence $\{\bm{y}_t\}$ is generated from a linear autoregressive process with additive noises,
\begin{equation}\label{eq:VARp}
\bm{y}_t=\bm{A}_1\bm{y}_{t-1}+\bm{A}_2\bm{y}_{t-2}+\cdots+\bm{A}_{t-P}\bm{y}_{t-P}+\bm{e}_t.
\end{equation}
Denote by $\bm{W}_0=(\bm{A}_1,\bm{A}_2,\dots,\bm{A}_P)$ the true parameters in \eqref{eq:VARp} and by $\cm{W}_0$ the corresponding folded tensor. 
We assume that $\cm{W}_0$ has Tucker ranks $r_1$, $r_2$ and $r_3$, and require the following conditions to hold.

\begin{condition}
	\label{cond:stationarity}
	All roots of matrix polynomial $|\bm{I}_N-\bm{A}_1z-\dots-\bm{A}_Pz^P|=0$ are outside
	unit circle.
\end{condition}

\begin{condition}
	\label{cond:error}
	The errors $\{\bm{e}_t\}$ is a sequence of independent Gaussian random vectors with mean zero and positive definite covariance matrix $\bm{\Sigma}_{\bm{e}}$, and $\bm{e}_t$ is independent of the historical observations $\bm{y}_{t-1},\bm{y}_{t-2},\cdots$ .
\end{condition}

Condition \ref{cond:stationarity} is sufficient and necessary for the strict stationarity of the linear autoregressive process. The Gaussian assumption in Condition \ref{cond:error} is very common in high-dimensional time series literature for technical convenience \citep{basu2015regularized}.

Multiple sequence data may exhibit strong temporal and inter-sequence dependence. To analyze how dependence in the data affects the learning efficiency, we follow \citet{basu2015regularized} to use the spectral measure of dependence below.

\begin{definition}
	Define the matrix polynomial $\mathcal{A}(z)=\bm{I}_N-\bm{A}_1z-\dots-\bm{A}_Pz^P$, where z is any point on the complex plane, and define its extreme eigenvalues as
	\begin{align*}
	\mu_{\min}(\mathcal{A}) &:=\underset{|z|=1}{\min} ~ \lambda_{\min}(\mathcal{A}^*(z)\mathcal{A}(z)),\\
	\mu_{\max}(\mathcal{A}) &:=\underset{|z|=1}{\max} ~ \lambda_{\max}(\mathcal{A}^*(z)\mathcal{A}(z)),
	\end{align*}
	where $\mathcal{A}^*(z)$ is the Hermitian transpose of $\mathcal{A}(z)$.
\end{definition}

By Condition \ref{cond:stationarity}, the extreme eigenvalues are bounded away from zero and infinity, $0<\mu_{\min}(\mathcal{A})\leq\mu_{\max}(\mathcal{A})<\infty$. Based on the spectral measure of dependence, we can derive the non-asymptotic statistical convergence rates for the LR and LTR estimators. Note that $C$ denotes a generic positive constant, which is independent of dimension and sample size, and may represent different values even on the same line. For any positive number $a$ and $b$, $a\lesssim b$ and $a\gtrsim b$ denote that there exists $C$ such that $a<Cb$ and $a>Cb$, respectively.

\begin{theorem}
	\label{thm:errorbound}
	Suppose that Conditions 1-2 are satisfied, and the sample size $T\gtrsim r_1r_2r_3+Nr_1+Nr_2+Pr_3$. With probability at least $1-\exp[-C(r_1r_2r_3+Nr_1+Nr_2+Pr_3)]-\exp(-C\sqrt{T})$,
	\begin{equation*}
	\|\cm{\widehat{W}}_{\textup{LTR}}-\cm{W}_0\|_{\textup{F}}\lesssim\mathcal{M}\sqrt{\frac{r_1r_2r_3+Nr_1+Nr_2+Pr_3}{T}},
	\end{equation*}
	where $\mathcal{M}:=[\lambda_{\max}(\bm{\Sigma}_{\bm{e}})\mu_{\max}(\mathcal{A})]/[\lambda_{\min}(\bm{\Sigma}_{\bm{e}})\mu^{1/2}_{\min}(\mathcal{A})]$ is the dependence measure constant.
\end{theorem}

\begin{theorem}
	\label{thm:errorbound2}
	Suppose that Conditions 1-2 are satisfied, $r\geq r_1$ and the sample size $T\gtrsim r(N+NP)$. With probability at least $1-\exp[-Cr(N+NP))]-\exp(-C\sqrt{T})$,
	\begin{equation*}
	\|\bm{\widehat{W}}_{\textup{LR}}-\bm{W}_{0}\|_{\textup{F}}\lesssim\mathcal{M}\sqrt{\frac{r(N+NP)}{T}}.
	\end{equation*}
\end{theorem}

The proofs of Theorems \ref{thm:errorbound} and \ref{thm:errorbound2} are provided in the supplemental material. The above two theorems present the non-asymptotic convergence upper bounds for LTR and LR estimators, respectively, with probability tending to one as the dimension and sample size grow to infinity. Both upper bounds take a general form of $\mathcal{M}\sqrt{d/T}$, where $\mathcal{M}$ captures the effect from dependence across $\bm{x}_t$, and $d$ denotes the number of parameters in Tucker decomposition or matrix factorization. From Theorems \ref{thm:errorbound} and \ref{thm:errorbound2}, we then can establish the sample complexity for these two estimators accordingly.

\begin{theorem}\label{thm:samplecomplexity}
	For a training error $\epsilon>0$, if the conditions of Theorem \ref{thm:errorbound} hold, then the sample complexity is $T\gtrsim (r_1r_2r_3+Nr_1+Nr_2+Pr_3)/\epsilon^2$ for the LTR estimator to achieve $\|\cm{\widehat{W}}_{\textup{LTR}}-\cm{W}_0\|_{\textup{F}}\leq\epsilon$.

	Moreover, if the conditions of Theorem \ref{thm:errorbound2} hold, then the sample complexity  is $T\gtrsim r(N+NP)/\epsilon^2$ for the LR estimator to achieve $\|\bm{\widehat{W}}_{\textup{LR}}-\bm{W}_{0}\|_{\textup{F}}\leq\epsilon$.
\end{theorem}

\begin{remark}
	The OLS estimator can be shown to have the convergence rate of $O(\sqrt{N^2P/T})$, and its sample complexity is $N^2P/\epsilon^2$ for a training error $\epsilon>0$ and $\|\bm{\widehat{W}}_{\textup{OLS}}-\bm{W}_{0}\|_{\textup{F}}\leq\epsilon$.
\end{remark}

The sample complexity for the linear autoregressive networks with different structures is proportional to the corresponding model complexity, i.e. sample complexity is $ O(\mathcal{M} d/\varepsilon^2) $. Compared with the OLS estimator, the LR and LTR estimators benefit from the compact low-dimensional structure and have smaller sample complexity. Among the three linear autoregressive networks, the LTR network has the most compact structure, and hence the smallest sample complexity.

\begin{remark}
	For the general supervised learning tasks rather than sequence modeling, the upper bound in Theorems \ref{thm:errorbound} and \ref{thm:errorbound2} can be extended to the case with independent and identically distributed $\bm{x}_t$, where $\mathcal{M}$ is replaced by the inverse of signal-to-noise ratios.
\end{remark}

The sample complexity analysis of the autoregressive networks can be extended to the general feed-forward networks, and explains why the low-rank structure can enhance the learning efficiency and reduce the sample complexity.

\section{Tucker Autoregressive Net}

This section introduces a compact autoregressive network by formulating the linear autoregressive network with the low-Tucker-rank structure \eqref{eq:lowrankVAR}, and it has a compact multi-layer CNN architecture. We call it the Tucker AutoRegressive (TAR) net for simplicity.

\subsection{Network Architecture}

Rather than directly constraining the matrix rank or Tucker ranks of weights in the zero-hidden-layer network, we can modify the network architecture by adding convolutional layers and fully connected layers to exploit low-rank structure. By some algebra, the framework \eqref{eq:lowrankVAR} can be rewritten into
\[
\bm{h}_t =\bm{U}_1\cm{G}_{(1)}\text{vec}(\bm{U}_2^\top\bm{X}_t\bm{U}_3)+\bm{b},
\]
where $\bm{X}_t=(\bm{y}_{t-1},\dots,\bm{y}_{t-P})$.
A direct translation of the low-Tucker-rank structure leads to a multi-layer convolutional network architecture with two convolutional layers and two fully connected layers; see Figure \ref{fig:CNN-direct} and Table \ref{tab:CNN}.

\begin{figure}[h]
	\centering
	\includegraphics[width=1\columnwidth]{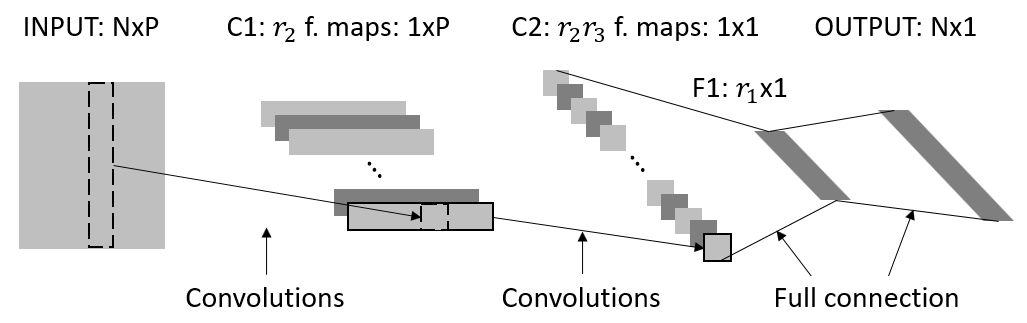}
	\caption{CNN structure of TAR net.}
	\label{fig:CNN-direct}
\end{figure}

\begin{table*}[t]
	\centering
	\begin{tabular}{ccccc}
		\toprule
		Symbol & Layer & Content and explanation & Dimensions & No. of parameters \\
		\midrule
		INPUT & - & design matrix & $N\times P$ & - \\
		C1 & $N\times1$ convolutions & $r_{2}$ feature maps & $1\times P$ & $Nr_{2}$ \\
		C2 & $1\times P$ convolutions & $r_{2}r_{3}$ feature maps & $1\times1$ & $Pr_{3}$ \\
		F1 & full connection & response factor loadings & $r_{1}\times1$ & $r_{1}r_{2}r_{3}$ \\
		OUTPUT & full connection & output prediction & $N\times1$ & $Nr_{1}$ \\
		\bottomrule
	\end{tabular}
	\caption{Specification of CNN structure in TAR net.}
	\label{tab:CNN}
\end{table*}

To be specific, each column in $\bm{U}_2$ is a $N\times 1$ convolution and the first layer outputs $r_2$ $1\times P$ feature maps. Similarly, $\bm{U}_3$ represents the convolution with kernel size $1\times P$ and $r_3$ channels. These two convolutional layers work as an encoder to extract the $r_2r_3$-dimensional representation of the $N\times P$ input $\bm{X}_t$ for predicting $\bm{y}_t$. Next, a full connection from $r_2r_3$ predictor features to $r_1$ output features with weights $\cm{G}_{(1)}$ is followed. Finally, a fully connected layer serves as a decoder to $N$ ouputs with weights $\bm{U}_1$.

The neural network architectures corresponding to the low-rank estimator $\bm{\widehat{W}}_{\textup{LR}}$ and ordinary least squares estimator without low-dimensional structure $\bm{\widehat{W}}_{\textup{OLS}}$ are the one-hidden-layer MLP with a bottleneck layer of size $r$ and the zero-hidden-layer fully connected network, respectively.

The CNN representation in Figure \ref{fig:CNN-direct} has a compact architecture with $r_1r_2r_3+Nr_1+Nr_2+Pr_3$ parameters, which is the same as that of the Tucker decomposition. Compared with the benchmark models, namely the one-hidden-layer MLP (MLP-1) and zero-hidden-layer MLP (MLP-0), the introduced low-Tucker-rank structure increases the depth of the network while reduces the total number of weights. When the Tucker ranks $r_1,r_2,r_3$ are small, the total number of parameters in our network is much smaller than those of the benchmark networks, which are $r(N+NP)$ and $N^2P$, respectively.

To capture the complicated and non-linear functional mapping between the prior inputs and future responses, non-linear activation functions, such as rectified linear unit (ReLU) or sigmoid function, can be added to each layer in the compact autoregressive network. Hence, the additional depth from transforming a low-Tucker-rank single layer to a multi-layer convolutional structure enables the network to better approximate the target function. The linear network without activation in the previous section can be called linear TAR net (LTAR).

\subsection{Separable Convolutional Kernels}

Separable convolutions have been extensively studied to replace or approximate large convolutional kernels by a series of smaller kernels. For example, this idea was explored in multiple iterations of the Inception blocks \citep{szegedy2015going,szegedy2016rethinking,szegedy2017inception} to decompose a convolutional layer with a $7\times 7$ kernel into that with $1\times 7$ and $7\times 1$ kernels.

Tensor decomposition is an effective method to obtain separable kernels. In our TAR net, these two convolutional layers extract the information from inputs along the column-wise direction and row-wise direction separately. Compared with the low-rank matrix structure, the additional decomposition in the Tucker decomposition along the second and third modes in fact segregates the full-sized convolutional kernel into $r_2r_3$ pairs of separable kernels.

\subsection{Two-Lane Network}

If no activation function is added, the first two row-wise and column-wise convolutional layers are exchangeable. However, exchanging these two layers with nonlinear activation functions can result in different nonlinear approximation and physical interpretation.

For the general case where we have no clear preference on the order of these two layers, we consider a two-lane network variant, called TAR-2 network, by introducing both structures into our model in parallel followed by an average pooling to enhance the flexibility; see Figure \ref{fig:LTRvsLTR-2}.

\begin{figure}[!htp]
	\centering
	\includegraphics[width=0.8\columnwidth]{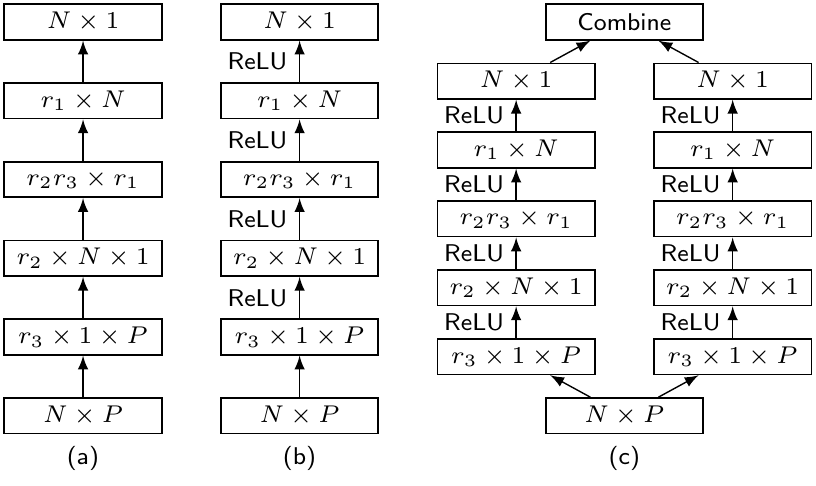}
	\caption{Structures of LTAR net (a), TAR net (b) and TAR-2 net (c).}
	\label{fig:LTRvsLTR-2}
\end{figure}

%
%


\subsection{Implementation}
\subsubsection{Details}
We implement our framework on PyTorch, and the Mean Squared Error (MSE) is the target loss function. The gradient descent method is employed for the optimization with learning rate and momentum being $0.01$ and $0.9$, respectively. If the loss function drops by less than $10^{-8}$, the procedure is then deemed to have reached convergence.

\subsubsection{Hyperparameter tuning}
In the TAR net, the sequential dependence range $P$ and the Tucker ranks $r_1$, $r_2$ and $r_3$ are prespecified hyperparameters. Since cross-validation cannot be applied to sequence modeling, we suggest tuning hyperparameters by grid search and rolling forecasting performance.

\section{Experiments}

This section first performs analysis on two synthetic datasets to verify the sample complexity established in Theorem \ref{thm:samplecomplexity} and to demonstrate the capability of TAR nets in nonlinear functional approximation.
A US macroeconomic dataset \cite{koop2013forecasting} is then analyzed by the TAR-2 and TAR nets, together with their linear counterparts. For the sake of comparison, some benchmark networks, including MLP-0, MLP-1, Recurrent Neural Network (RNN) and Long Short-Term Memory (LSTM), are also applied to the dataset.
\subsection{Numerical Analysis for Sample Complexity}
\subsubsection{Settings}
In TAR net or the low-Tucker-rank framework \eqref{eq:lowrankVAR}, the hyperparameters, $r_1, r_2$ and $r_3$, are of significantly smaller magnitude than $N$ or $P$, and are equally set to $2$ or $3$. As sample complexity is of prime interest rather than the range of sequential dependence, we let $P$ equal to $3,5$ or $8$. For each combination of $(r_1,r_2,r_3,P)$, we consider $N = 9,25$ and $36$, and the sample size $T$ is chosen such that $\sqrt{N/T} = (0.15, 0.25, 0.35, 0.45)$.

\subsubsection{Data generation}
We first generate a core tensor $\cm{G} \in \mathbb{R}^{r_1\times r_2\times r_3}$ with entries being independent standard normal random variables, and then rescale it such that the largest singular value of $\cm{G}_{(1)}$ is $0.9$.
For each $1\leq i\leq 3$, the leading $r_{i}$ singular vectors of random standard Gaussian matrices are used to form $\bm{{U}}_i$. The weight tensor $\cm{W}_0$ can thereby be reconstructed, and it is further rescaled to satisfy Condition \ref{cond:stationarity}. We generate $200$ sequences with identical $\cm{W}_0$. The first $500$ simulated data points at each sequence are discarded to alleviate the influence of the initial values. We apply the MLP-0, MLP-1 and LTAR to the synthetic dataset. The averaged estimation errors for the corresponding OLS, LR, and LTR estimators are presented in Figure \ref{fig:Thm1}.

\subsubsection{Results}
The $x$-axis in Figure \ref{fig:Thm1} represents the ratio of $\sqrt{N/T}$, and the $y$-axis represents the averaged estimation error in Frobenius norm. Along each line, as $N$ is set to be fixed, we obtain different points by readjusting the sample size $T$. Roughly speaking, regardless of the models and parameter settings, estimation error increases with varying rates as the sample size decreases. The rates for OLS rapidly become explosive, followed by LR, whereas LTR remains approximately linear, which is consistent with our findings at Theorem \ref{thm:samplecomplexity}.

Further observation reveals that the increase in $P$ predominantly accelerates the rates for OLS and LR, but appears to have insignificant influence on the estimation error from LTR.

For the case with $P=8$, instability of the estimation error manifests itself in LR under insufficient sample size, say when $\sqrt{N/T}$ is as large as $0.35$. This further provides the rationale for dimension reduction along sequential order. When $\sqrt{N/T} = 0.45$, the solution is not unique for both OLS and LR, and consequently, these points are not shown in the figure.

\subsection{Numerical Analysis for Nonlinear Approximation}

\subsubsection{Settings} The target of this experiment is to compare the expressiveness of LTAR, TAR and TAR-2 nets. The conjecture is that, regardless of the data generating process, TAR-2 and TAR nets under the same hyperparameter settings as the LTAR net would have an elevated ability to capture nonlinear features. We set $(r_1, r_2, r_3, N, P) = (2,2,2,25,3)$, and have also tried several other combinations. Similar findings can be observed, and the results are hence omitted here.

\subsubsection{Data generation} Two data generating processes are considered to create sequences with either strictly linear or highly nonlinear features in the embedding feature space. We refer to them as L-DGP and NL-DGP, respectively. L-DGP is achieved by randomly assigning weights to LTAR layers and producing a recursive sequence with a given initial input matrix. NL-DGP is attained through imposing a nonlinear functional transformation to the low-rank hidden layer of an MLP. In detail, we first transformed a $N\times P$ matrix to a $r_1\times r_2$ low-rank encoder. Then, we applied a nonlinear mapping $f(\bm{\cdot}) = \cos(1/\norm{\bm{\cdot}}_\text{F})$ to the encoder, before going through a fully connected layer to retrieve an output of size $N\times 1$.


\subsubsection{Implementation \& Evaluation}
In this experiment, we use L-DGP and NL-DGP to separately generate 200 data sequences which are fitted by TAR-2, TAR and LTAR nets. The sequence lengths are chosen to be either $101$ or $501$. For each sequence, the last data point is retained as a single test point, whereas the rest are used in model training. We adopt three evaluation metrics, namely, the averaged $L_2$ norm between prediction and true value, the standard Root-Mean-Square Error (RMSE), and Mean Absolute Error (MAE). The results are given in Table \ref{tab:sim2}.

\subsubsection{Results}
When the data generating process is linear (L-DGP), the LTAR net reasonably excels in comparison to the other two, obtaining the smallest $L_2$-norm, RMSE and MAP. TAR-2 yields poorer results for a small sample size of $100$ due to possible overparametrization. However, its elevated expressiveness leads it to outperform TAR when $T=500$.

For nonlinear data generating process (NL-DGP), as we expect, the TAR-2 and TAR nets with nonlinear structure outperform the LTAR net. In the meanwhile, as the exchangeability of latent features holds, the TAR-2 net seems to suffer from model redundancy and thereby performs worse than the TAR net.

\begin{figure}
	\centering
	\includegraphics[width=0.7\columnwidth]{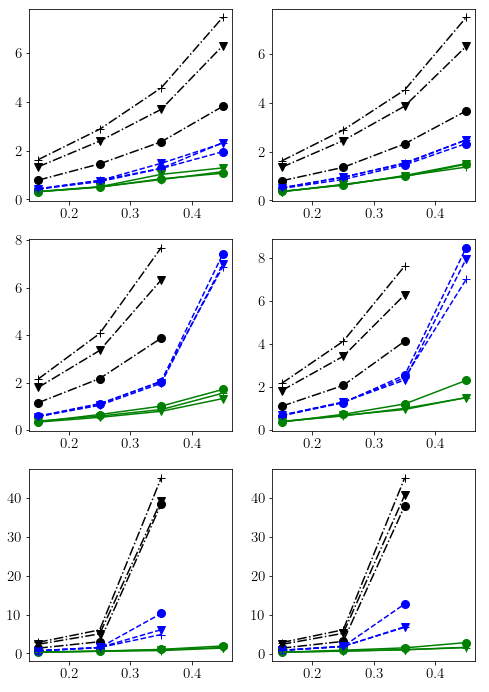}
	\caption{Experiment on sample complexity. Results are shown for OLS (\protect\blackline), LR (\protect\blueline) and LTR (\protect\greenline) estimators. Three different values of $N$ are presented by different markers: $ N = 9 $ ($\bullet $), $ N = 25 $ ($\blacktriangledown$) and $ N = 36 $ (+). We set $(r_1,r_2,r_3)=(2,2,2)$ for the three subplots in the left column and $(r_1,r_2,r_3)=(3,3,3)$ for subplots in the right column. And the upper, middle and lower panels refer to cases with $P = 3, 5$ and $9$, respectively.}
	\label{fig:Thm1}
\end{figure}

\subsection{US Macroeconomic Dataset}
\subsubsection{Dataset} We use the dataset provided in \citet{koop2013forecasting} with $40$ US macroeconomic variables. They cover various aspects of financial and industrial activities, including consumption, production indices, stock market indicators and the interest rates. The data series are taken quarterly from 1959 to 2007 with a total of $194$ observed time points. In the preprocessing step, the series were transformed to be stationary before being standardized to have zero mean and unit variance; details see the supplemental material.

\begin{table}[h]
	\centering
	\begin{tabular}{cccccc}
		\toprule
		DGP&$T$&Network& $L_2$-norm &RMSE&MAP\\
		\midrule
		\multirow{6}{*}{\makecell{L-DGP}} & \multirow{3}{*}{100} &TAR-2&5.5060&1.1238&0.8865\\
		&&TAR&5.4289&1.0998&0.8702\\
		&&LTAR& \textbf{5.1378} & \textbf{1.0388} & \textbf{0.8265}\\[1mm]

		& \multirow{3}{*}{500} &TAR-2&5.1836&1.0493&0.8369\\
		&&TAR&5.2241&1.0585&0.8436\\
		&&LTAR &\textbf{4.9338} & \textbf{0.9972} &\textbf{0.7936}\\
		\midrule
		\multirow{6}{*}{\makecell{NL-DGP}} & \multirow{3}{*}{100} &TAR-2&5.2731& \textbf{1.0703} &0.8579 \\
		&&TAR& \textbf{5.2710} &1.0712&\textbf{0.8510}\\
		&&LTAR&5.3161&1.0738&0.8573\\[1mm]
		& \multirow{3}{*}{500} &TAR-2&5.0084&1.0111&0.8062 \\
		&&TAR& \textbf{5.0036} & \textbf{1.0110} &\textbf{0.8060}\\
		&&LTAR&5.0144&1.0126&0.8087\\
		\bottomrule
	\end{tabular}
	\caption{Performance of different networks on fitting L-DGP and NL-DGP datasets.}
	\label{tab:sim2}
\end{table}

\subsubsection{Models for comparison}
For the sake of comparison, besides the proposed models, TAR-2, TAR and LTAR, we also consider four other commonly used networks in the literature with well-tuned hyperparameters. The first two are the previously mentioned MLP-0 and MLP-1. The remaining two are RNN and LSTM, which are two traditional sequence modeling frameworks.
RNN implies an autoregressive moving average framework and can transmit extra useful information through the hidden layers. It is hence expected to outperform an autoregressive network. LSTM may be more susceptible to small sample size. As a result, RNN and LSTM with the optimal tuning hyperparameter serve as our benchmarks.
\subsubsection{Implementation}
Following the settings in \citet{koop2013forecasting}, we set $P = 4$. Consistently, the sequence length in both RNN and LSTM is fixed to be $4$, and we consider only one hidden layer. The number of neurons in the hidden layer is treated as a tunable hyperparameter. To be on an equal footing with our model, the size of the hidden layer in MLP-1 is set to $4$. We further set $r_1 = 4$, $r_2 = 3$ and $r_3 = 2$. The bias terms are added back to the TAR-2, TAR and LTAR nets for expansion of the model space.

The dataset is segregated into two subsets: the first 104 time points of each series are used as the training samples with an effective sample size of 100, whereas the rolling forecast procedure is applied to the rest 90 test samples. For each network, one-step-ahead forecasting is carried out in a recursive fashion. In other words, the trained network predicts one future step, and immediately includes the new observation for the prediction of the next step. The averaged $L_2$-norm, RMSE and MAP are used as the evaluation criteria.

\begin{table}
	\centering
	\begin{tabular}{cccc}
		\toprule
		Network&$L_2$-norm&RMSE&MAE\\
		\midrule
		MLP-0 &11.126&1.8867&1.3804\\
		MLP-1 &7.8444&1.3462&1.0183 	\\
		RNN &5.5751&0.9217&0.7064	\\
		LSTM &5.8274&0.9816&0.7370	\\ [1mm]
		\textbf{LTAR} &5.5257&0.9292&0.6857	\\
		\textbf{TAR} &5.4675&0.9104&0.6828 	\\
		\textbf{TAR-2} & \textbf{5.4287} & \textbf{0.8958} & \textbf{0.6758}	\\
		\bottomrule
	\end{tabular}
	\caption{Performance comparison on US macroeconomic dataset.}
	\label{tab:Real}
\end{table}

\subsubsection{Results}
From Table \ref{tab:Real}, the proposed TAR-2 and TAR nets rank top two in terms of one-step-ahead rolling forecast performance, exceeding the fine-tuned RNN model with the size of the hidden layer equal to one. The two-lane network TAR-2 clearly outperforms the one-lane network TAR emphasizing its ability to capture non-exchangeable latent features. According to our experiments, the performance of both RNN and LSTM deteriorates as the dimension of the hidden layer increases, which indicates that overfitting is a serious issue for these two predominate sequence modeling techniques. Figure \ref{fig:real} plots the forecast values against the true values of the variables ``SEYGT10" (the spread between 10-yrs and 3-mths treasury bill rates) for the TAR-2 net and the RNN model. It can be seen that TAR-2 shows strength in capturing the pattern of peaks and troughs, and hence resembles the truth more closely.
\begin{center}
	\includegraphics[width=0.8\columnwidth]{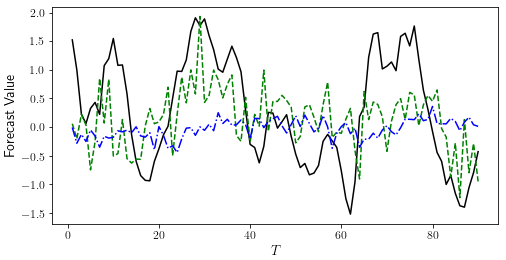}
	\captionof{figure}{Rolling forecasting for the variable ``SEYGT10". The solid line (\protect\blacklinereal), the dashed line (\protect\greenlinereal) and the dash-dotted line (\protect\bluelinereal) represent true value, predictions from TAR-2 and RNN, respectively.}
	\label{fig:real}
\end{center}
\section{Conclusion and Discussion}
This paper rearranges the weights of an autoregressive network into a tensor, and then makes use of the Tucker decomposition to introduce a low-dimensional structure. A compact autoregressive network is hence proposed to handle the sequences with long-range dependence. Its sample complexity is also studied theoretically. The proposed network can achieve better prediction performance on a macroeconomic dataset than some state-of-the-art methods including RNN and LSTM.

For future research, this work can be improved in three directions. First, our sample complexity analysis is limited to linear models, and it is desirable to extend the analysis to networks with nonlinear activation functions. Secondly, the dilated convolution, proposed by WaveNet \cite{oord2016wavenet}, can reduce the convolutional kernel size along the sequential order, and hence can efficiently access the long-range historical inputs. This structure can be easily incorporated into our framework to further compress the network. Finally, a deeper autoregressive network can be constructed by adding more layers into the current network to enhance the expressiveness of nonlinearity.

\appendix
\section{Technical Proofs of Theorems 1 and 2}

\subsection{Proofs of Theorems}

	We first prove the statistical convergence rate for $\cm{\widehat{W}}_{\text{LTR}}$ and denote it as $\cm{\widehat{W}}$ for simplicity. The main ideas of the proof come from \citet{raskutti2019convex}.

	Denote $\bm{\Delta}=\cm{\widehat{W}}-\cm{W}_0$, then by the optimality of the LTR estimator,
	\begin{equation*}\begin{split}
	&\frac{1}{T}\sum_{t=1}^T\|\bm{y}_t-\cm{\widehat{W}}_{(1)}\bm{x}_t\|_2^2\leq\frac{1}{T}\sum_{t=1}^T\|\bm{y}_t-\bm{W}_{0}\bm{x}_{t}\|_2^2\\
	\Rightarrow~&\frac{1}{T}\sum_{t=1}^T\|\bm{\Delta}_{(1)}\bm{x}_{t}\|_2^2\leq\frac{2}{T}\sum_{t=1}^T\langle\bm{e}_t,\bm{\Delta}_{(1)}\bm{x}_{t}\rangle\\
	\Rightarrow~&\frac{1}{T}\sum_{t=1}^T\|\bm{\Delta}_{(1)}\bm{x}_{t}\|_2^2\leq\frac{2}{T}\left\langle \sum_{t=1}^T\bm{e}_t\circ\bm{X}_{t},\bm{\Delta}\right\rangle,
	\end{split}\end{equation*}
	where $\circ$ denotes the tensor outer product.

	Since the Tucker ranks of both $\cm{\widehat{W}}$ and $\cm{W}$ are $(r_1,r_2,r_3)$, the Tucker ranks of $\bm{\Delta}$ are at most $(2r_1,2r_2,2r_3)$. Denote the set of tensor $\mathcal{S}(r_1,r_2,r_3)=\{\cm{W}\in\mathbb{R}^{N\times N\times P}:\|\cm{W}\|_{\textup{F}}=1,~\textup{rank}_i(\cm{W})\leq r_i,~i=1,2,3\}$. Then, we have
	\begin{equation*}
	\frac{1}{T}\sum_{t=1}^T\|\bm{\Delta}_{(1)}\bm{x}_t\|_2^2\leq2\|\bm{\Delta}\|_{\text{F}}\sup_{\scriptsize{\cm{W}}\in\mathcal{S}(2r_1,2r_2,2r_3)}\left\langle\frac{1}{T}\sum_{t=1}^T\bm{e}_t\circ\bm{X}_t,\cm{W}\right\rangle.
	\end{equation*}

	Given the restricted strong convexity condition, namely $\alpha_{\textup{RSC}}\|\bm{\Delta}\|_{\textup{F}}^2\leq T^{-1}\sum_{t=1}^T\|\bm{\Delta}_{(1)}\bm{x}_t\|_2^2$, we can obtain an upper bound,
	\begin{equation*}
	\|\bm{\Delta}\|_{\text{F}}\leq\frac{2}{\alpha_{\text{RSC}}}\sup_{\scriptsize{\cm{W}}\in\mathcal{S}(2r_1,2r_2,2r_3)}\left\langle\frac{1}{T}\sum_{t=1}^T\bm{e}_t\circ\bm{X}_t,\cm{W}\right\rangle.
	\end{equation*}

	Since $\bm{y}_t$ is a strictly-stationary VAR process, we can easily check that it is a $\beta$-mixing process. Denote the unconditional covariance matrix of $\bm{x}_t$ as $\bm{\Sigma}_{\bm{x}}$. Let $\bm{m}_t=\bm{\Sigma}_{\bm{x}}^{-1/2}\bm{x}_t$ and $\bm{M}_t\in\mathbb{R}^{N\times P}$ be the corresponding matrix from $\bm{m}_t$. By spectral measure \citep[Proposition 2.3]{basu2015regularized}, the largest eigenvalue of $\bm{\Sigma}_{\bm{x}}$ is upper bounded, namely $\lambda_{\max}(\bm{\Sigma}_{\bm{x}})\leq2\pi\mathcal{M}(f_{\bm{X}})\leq\lambda_{\max}(\bm{\Sigma}_{\bm{e}})/\mu_{\min}(\mathcal{A})$.

	Therefore, conditioning on all $(\bm{e}_t)$'s, we have
	\begin{equation*}\begin{split}
	&\mathbb{P}\left\{\sup_{\scriptsize{\cm{W}}\in\mathcal{S}(2r_1,2r_2,2r_3)}\frac{1}{T}\sum_{t=1}^T\langle\bm{e}_t\circ\bm{X}_{t},\cm{W}\rangle>x\right\}\\
	\leq&\mathbb{P}\left\{\sup_{\scriptsize{\cm{W}}\in\mathcal{S}(2r_1,2r_2,2r_3)}\frac{1}{T}\sum_{t=1}^T\langle\bm{e}_t\circ\bm{M}_t,\cm{W}\rangle>\sqrt{\frac{\mu_{\min}(\mathcal{A})}{\lambda_{\max}(\bm{\Sigma}_{\bm{e}})}}x\right\}.
	\end{split}\end{equation*}

	Denote $\bm{M}=(\bm{m}_1,\dots,\bm{m}_T)$ and denote $\bm{m}^{(i)}$ as the $i$-th row of $\bm{M}$. Further, if we condition on $\{\bm{M}_t\}$, since $\bm{e}_t$ is a sequence of iid random vectors with mean zero and covariance $\bm{\Sigma}_{\bm{e}}$, for any $x>0$,
	\begin{equation*}\begin{split}
	&\mathbb{P}\left\{\sup_{\scriptsize{\cm{W}}\in\mathcal{S}(2r_1,2r_2,2r_3)}\frac{1}{T}\sum_{t=1}^T\langle\bm{e}_t\circ\bm{M}_t,\cm{W}\rangle>x\right\}\\
	\leq&\mathbb{P}\left\{\sup_{\scriptsize{\cm{W}}\in\mathcal{S}(2r_1,2r_2,2r_3)}\max_{1\leq j\leq NP}\sqrt{\frac{\|\bm{m}^{(j)}\|_2^2}{T}}\sqrt{\frac{\lambda_{\max}(\bm{\Sigma}_{\bm{e}})}{T}}\langle\cm{N},\cm{W}\rangle>x\right\},
	\end{split}\end{equation*}
	where $\cm{N}\in\mathbb{R}^{N\times N\times P}$ is a random tensor with i.i.d. standard normal entries.

	Since $\bm{m}^{(j)}$ is a sequence of $\beta$-mixing random variables with mean zero and unit variance, by Lemma \ref{lemma:beta_mixing}, for each $j$, there exists some constant $C$ and $c>0$, such that
	\begin{equation*}
	\mathbb{P}\left[\frac{\|\bm{m}^{(j)}\|_2^2}{T}\geq4\right]\leq C\sqrt{T}\exp(-c\sqrt{T}).
	\end{equation*}
	Taking a union bound, if $\sqrt{T}\gtrsim\log(NP)$, we have
	\begin{equation*}
	\mathbb{P}\left[\max_{1\leq j\leq NP}\frac{\|\bm{m}^{(j)}\|_2}{\sqrt{T}}\leq2\right]\geq1-C\exp[c(\log(T)+\log(NP)-\sqrt{T})]\geq1-C\exp(-c\sqrt{T}).
	\end{equation*}

	Therefore, we have
	\begin{equation*}\begin{split}
	&\mathbb{P}\left[\sup_{\scriptsize{\cm{W}}\in\mathcal{S}(2r_1,2r_2,2r_3)}\frac{1}{T}\sum_{t=1}^T\langle\bm{e}_t\circ\bm{X}_{t},\cm{W}\rangle>x\right]\\
	\leq&\mathbb{P}\left[\sup_{\scriptsize{\cm{W}}\in\mathcal{S}(2r_1,2r_2,2r_3)}\frac{1}{T}\sum_{t=1}^T\langle\bm{e}_t\circ\bm{M}_t,\cm{W}\rangle>\sqrt{\frac{\mu_{\min}(\mathcal{A})}{\lambda_{\max}(\bm{\Sigma}_{\bm{e}})}}x\right]\\
	\leq&\mathbb{P}\left[\sup_{\scriptsize{\cm{W}}\in\mathcal{S}(2r_1,2r_2,2r_3)}2\sqrt{\frac{\lambda_{\max}(\bm{\Sigma}_{\bm{e}})}{T}}\langle\cm{N},\cm{W}\rangle\geq\sqrt{\frac{\mu_{\min}(\mathcal{A})}{\lambda_{\max}(\bm{\Sigma}_{\bm{e}})}}x\right]+C\exp(-c\sqrt{T})\\
	=&\mathbb{P}\left[\sup_{\scriptsize{\cm{W}}\in\mathcal{S}(2r_1,2r_2,2r_3)}\frac{1}{\sqrt{T}}\langle\cm{N},\cm{W}\rangle\geq\frac{\sqrt{\mu_{\min}(\mathcal{A})}x}{2\lambda_{\max}(\bm{\Sigma}_{\bm{e}})}\right]+C\exp(-c\sqrt{T}).
	\end{split}\end{equation*}

	For any fixed $\cm{W}\in\mathcal{S}(2r_1,2r_2,2r_3)$, it can be checked that $\langle\cm{N},\cm{W}\rangle\sim N(0,1)$. Hence, there exists a constant $C$ such that for any $t>0$
	\begin{equation*}
	\mathbb{P}\left[\langle\cm{N},\cm{A}\rangle\geq t\right]\leq\exp(-Ct^2).
	\end{equation*}

	Consider a $\epsilon$-net $\overline{\mathcal{S}}(2r_1,2r_2,2r_3)$ for $\mathcal{S}(2r_1,2r_2,2r_3)$. Then, for any $\cm{W}\in\mathcal{S}(2r_1,2r_2,2r_3)$, there exists a $\cm{\overline{W}}\in\overline{\mathcal{S}}(2r_1,2r_2,2r_3)$ such that $\|\cm{\overline{W}}-\cm{W}\|_{\text{F}}\leq\epsilon$. Note that the multilinear ranks of $\bm{\overline{\Delta}}=\cm{\overline{W}}-\cm{W}$ are at most $(4r_1,4r_2,4r_3)$. As shown in Figure \ref{fig:tucker8}, we can split the higher order singular value decomposition (HOSVD) of $\bm{\overline{\Delta}}$ into 8 parts such that $\bm{\overline{\Delta}}=\sum_{i=1}^8\bm{\Delta}_i$ such that $\text{rank}_j(\bm{\Delta}_i)\leq 2r_j$ for $i=1,\dots,8$ and $j=1,2,3$, and $\langle\bm{\Delta}_j,\bm{\Delta}_k\rangle=0$ for any $j\neq k$.

	\begin{figure}[!htbp]
		\begin{center}
			\includegraphics[width=8cm]{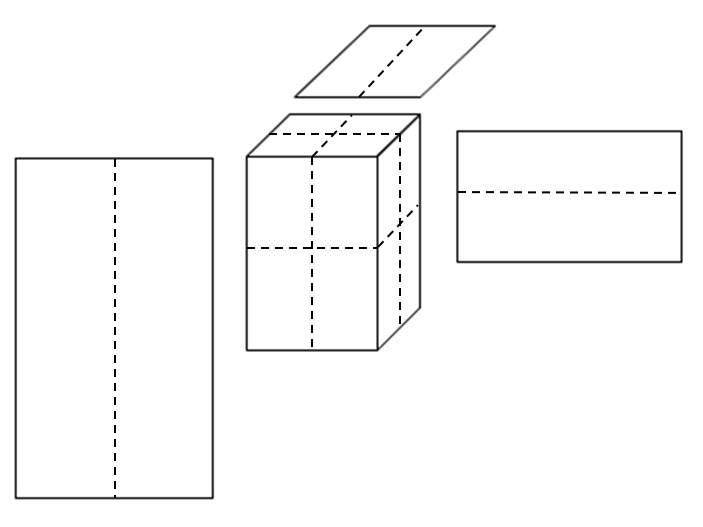}
			\caption{Splitting HOSVD}
			\label{fig:tucker8}
		\end{center}
	\end{figure}

	Note that
	\begin{equation*}
	\langle\cm{N},\cm{W}\rangle=\langle\cm{N},\cm{\overline{W}}\rangle+\sum_{i=1}^8\langle\cm{N},\bm{\overline{\Delta}}_i\rangle=\langle\cm{N},\cm{\overline{W}}\rangle+\sum_{i=1}^8\langle\cm{N},\bm{\overline{\Delta}}_i/\|\bm{\overline{\Delta}}_i\|_{\text{F}}\rangle\|\bm{\overline{\Delta}}_i\|_{\text{F}}.
	\end{equation*}
	Since each $\bm{\overline{\Delta}}_i/\|\bm{\overline{\Delta}}_i\|_{\text{F}}\in\mathcal{S}(2r_1,2r_2,2r_3)$, $\langle\cm{N},\bm{\overline{\Delta}}_i/\|\bm{\overline{\Delta}}_i\|_{\text{F}}\rangle\leq\sup_{\cm{A}\in\mathcal{S}(2r_1,2r_2,2r_3)}\langle\cm{N},\cm{A}\rangle$.

	Since $\|\bm{\overline{\Delta}}\|_\text{F}^2=\sum_{i=1}^8\|\bm{\overline{\Delta}}_i\|_\text{F}^2$, by Cauchy inequality, $\sum_{i=1}^8\|\bm{\overline{\Delta}}_i\|_{\text{F}}\leq2\sqrt{2}\|\bm{\overline{\Delta}}\|_{\text{F}}\leq2\sqrt{2}\epsilon$. Hence, we have
	\begin{equation*}\begin{split}
	\gamma:=\sup_{\scriptsize{\cm{W}}\in\mathcal{S}(2r_1,2r_2,2r_3)}\langle\cm{N},\cm{W}\rangle\leq\max_{\scriptsize{\cm{\overline{W}}}\in\overline{\mathcal{S}}(2r_1,2r_2,2r_3)}\langle\cm{N},\cm{\overline{W}}\rangle+2\sqrt{2}\gamma\epsilon.
	\end{split}\end{equation*}
	In other words,
	\begin{equation*}
	\sup_{\scriptsize{\cm{W}}\in\mathcal{S}(2r_1,2r_2,2r_3)}\langle\cm{N},\cm{W}\rangle\leq(1-2\sqrt{2}\epsilon)^{-1}\max_{\scriptsize{\cm{\overline{W}}}\in\overline{\mathcal{S}}(2r_1,2r_2,2r_3)}\langle\cm{N},\cm{\overline{W}}\rangle.
	\end{equation*}
	Therefore, we have
	\begin{equation*}\begin{split}
	&\mathbb{P}\left[\sup_{\scriptsize{\cm{W}}\in\mathcal{S}(2r_1,2r_2,2r_3)}\frac{1}{\sqrt{T}}\langle\cm{N},\cm{W}\rangle\geq\frac{\sqrt{\mu_{\min}(\mathcal{A})}x}{2\lambda_{\max}(\bm{\Sigma}_{\bm{e}})}\right]\\
	\leq&\mathbb{P}\left[\max_{\scriptsize{\cm{W}}\in\overline{\mathcal{S}}(2r_1,2r_2,2r_3)}\frac{1}{\sqrt{T}}\langle\cm{N},\cm{W}\rangle\geq\frac{(1-2\sqrt{2}\epsilon)\sqrt{\mu_{\min}(\mathcal{A})}x}{2\lambda_{\max}(\bm{\Sigma}_{\bm{e}})}\right]\\
	\leq&|\overline{\mathcal{S}}(2r_1,2r_2,2r_3)|\cdot\mathbb{P}\left[\langle\cm{N},\cm{W}\rangle\geq\frac{(1-2\sqrt{2}\epsilon)\sqrt{T\mu_{\min}(\mathcal{A})}x}{2\lambda_{\max}(\bm{\Sigma}_{\bm{e}})}\right]\\
	\leq&\exp\left[\log(|\overline{\mathcal{S}}(2r_1,2r_2,2r_3)|)-\frac{C(1-2\sqrt{2}\epsilon)^2T\mu_{\min}(\mathcal{A})x^2}{\lambda^2_{\max}(\bm{\Sigma}_{\bm{e}})}\right].
	\end{split}\end{equation*}

	By Lemma \ref{lemma:covering}, $|\overline{\mathcal{S}}(r_1,r_2,r_3)|\leq(12/\epsilon)^{8r_1r_2r_3+2Nr_1+2Nr_2+2Pr_3}$. We can take $\epsilon=1/10$ and $x=C\sqrt{r_1r_2r_3+Nr_1+Nr_2+Pr_3}\lambda_{\max}(\bm{\Sigma}_{\bm{e}})T^{-1/2}\mu^{-1/2}_{\min}(\mathcal{A})$, and then obtain that
	\begin{equation*}\begin{split}
	&\mathbb{P}\left[\sup_{\scriptsize{\cm{W}}\in\mathcal{S}(2r_1,2r_2,2r_3)}\frac{1}{T}\sum_{t=1}^T\langle\bm{e}_t\circ\bm{X}_{t},\cm{W}\rangle\geq\frac{C\lambda_{\max}(\bm{\Sigma}_{\bm{e}})}{\mu^{1/2}_{\min}(\mathcal{A})}\sqrt{\frac{r_1r_2r_3+Nr_1+Nr_2+Pr_3}{T}}\right]\\
	\leq&1-\exp[-C(r_1r_2r_3+Nr_1+Nr_2+Pr_3)]-\exp(-C\sqrt{T}).
	\end{split}\end{equation*}

	Finally, by Lemma \ref{lemma:RSC}, the restricted convexity condition holds for $\alpha_{\text{RSC}}=\lambda_{\min}(\bm{\Sigma}_{\bm{e}})/(2\mu_{\max}(\mathcal{A}))$, with probability at least $1-C\exp[-c(r_1r_2r_3+Nr_1+Nr_2+Pr_3)]$, which concludes the proof of $\cm{\widehat{W}}_{\text{LTR}}$.

	Similarly, we can obtain the error upper bound for $\cm{\widehat{W}}_{\text{LT}}$ by replacing the covering number of low-Tucker-rank tensors to that of low-rank matrices, and the covering number of low-rank matrices are investigated by \citet{candes2011tight}.

\subsection{Three Lemmas Used in the Proofs of Theorems}

\begin{lemma} \label{lemma:covering} (Covering number of low-multilinear-rank tensors) The $\epsilon$-covering number of the set $\mathcal{S}(r_1,r_2,r_3):=\{\cm{T}\in\mathbb{R}^{p_1\times p_2\times p_3}:\|\cm{T}\|_{\textup{F}}=1,\textup{rank}_i(\cm{T}_{(i)})\leq r_i,~i=1,2,3\}$ is
	\begin{equation*}
	|\overline{\mathcal{S}}(r_1,r_2,r_3)|\leq(12/\epsilon)^{r_1r_2r_3+p_1r_1+p_2r_2+p_3r_3}.
	\end{equation*}
\end{lemma}

\begin{proof}
	The proof hinges on the covering number for the low-rank matrix studied by \citet{candes2011tight}.

	Recall the HOSVD $\cm{T}=[\![\cm{G};\bm{U}_1,\bm{U}_2,\bm{U}_3]\!]$, where $\|\cm{T}\|_{\text{F}}=1$ and each $\bm{U}_i$ is an orthonormal matrix. We construct an $\epsilon$-net for $\cm{T}$ by covering the set of $\cm{G}$ and all $\bm{U}_i$'s. We take $\overline{G}$ to be an $\epsilon/4$-net for $\cm{G}$ with $|\overline{G}|\leq(12/\epsilon)^{r_1r_2r_3}$. Next, let $O_{n,r}=\{\bm{U}\in\mathbb{R}^{n\times r}: \bm{U}^\top\bm{U}=I_{r}\}$. To cover $O_{n,r}$, it is beneficial to use the $\|\cdot\|_{1,2}$ norm, defined as
	\begin{equation*}
	\|\bm{X}\|_{1,2}=\max_{i}\|\bm{X}_i\|_2,
	\end{equation*}
	where $\bm{X}_i$ denotes the $i$th column of $\bm{X}$. Let $Q_{n,r}=\{\bm{X}\in\mathbb{R}^{n\times r}:\|\bm{X}\|_{1,2}\leq1\}$. One can easily check that $O_{n,r}\subset Q_{n,r}$, and thus an $\epsilon/4$-net $\overline{O}_{n,r}$ for $O_{n,r}$ obeying $|\overline{O}_{n,r}|\leq(12/\epsilon)^{nr}$.

	Denote $\overline{T}=\{[\![\cm{\overline{G}};\bm{\overline{U}}_1,\bm{\overline{U}}_2,\bm{\overline{U}}_3]\!]:~\cm{\overline{G}}\in\overline{G},~\bm{\overline{U}}_i\in \overline{O}_{n_i,r_i},~i=1,2,3\}$ and we have $|\overline{T}|\leq|\overline{G}|\times|\overline{O}_{N\times r_1}|\times|\overline{O}_{N\times r_2}|\times|\overline{O}_{P\times r_3}|=(12/\epsilon)^{r_1r_2r_3+Nr_1+Nr_2+Pr_3}$. It suffices to show that for any $\cm{T}\in\mathcal{S}(r_1,r_2,r_3)$, there exists a $\cm{\overline{T}}\in\overline{T}$ such that $\|\cm{T}-\cm{\overline{T}}\|_{\text{F}}\leq\epsilon.$

	For any fixed $\cm{T}\in\mathcal{S}(r_1,r_2,r_3)$, decompose it by HOSVD as $\cm{T}=[\![\cm{G};\bm{U}_1,\bm{U}_2,\bm{U}_3]\!]$. Then, there exist $\cm{\overline{T}}=[\![\cm{\overline{G}};\bm{\overline{U}}_1,\bm{\overline{U}}_2,\bm{\overline{U}}_3]\!]$with $\cm{\overline{G}}\in\overline{G}$, $\bm{\overline{U}}_i\in\overline{O}_{n_i,r_i}$ satisfying that $\|\bm{U}_i-\bm{\overline{U}}_i\|_{1,2}\leq\epsilon/4$ and $\|\cm{G}-\cm{\overline{G}}\|_{\text{F}}\leq\epsilon/4$. This gives
	\begin{equation*}\begin{split}
	\|\cm{T}-\cm{\overline{T}}\|_{\text{F}}&\leq\|[\![\cm{G}-\cm{\overline{G}};\bm{U}_1,\bm{U}_2,\bm{U}_3]\!]\|_{\text{F}}+\|[\![\cm{\overline{G}};\bm{U}_1-\bm{\overline{U}}_1,\bm{U}_2,\bm{U}_3]\!]\|_{\text{F}}\\
	&+\|[\![\cm{\overline{G}};\bm{\overline{U}}_1,\bm{U}_2-\bm{\overline{U}}_2,\bm{U}_3]\!]\|_{\text{F}}+\|[\![\cm{\overline{G}};\bm{\overline{U}}_1,\bm{\overline{U}}_2,\bm{U}_3-\bm{\overline{U}}_3]\!]\|_{\text{F}}.
	\end{split}\end{equation*}

	Since each $\bm{U}_i$ is an orthonormal matrix, the first term is $\|\cm{G}-\cm{\overline{G}}\|_{\text{F}}\leq \epsilon/4$. For the second term, by the all-orthogonal property of $\cm{\overline{G}}$ and the orthonormal property of $\bm{U}_2$ and $\bm{U}_3$,
	\begin{equation*}
	\|[\![\cm{\overline{G}};\bm{U}_1-\bm{\overline{U}}_1,\bm{U}_2,\bm{U}_3]\!]\|_{\text{F}}
	=\|\cm{\overline{G}}\times_1(\bm{U}_1-\bm{\overline{U}}_1)\|_{\text{F}}\leq\|\cm{\overline{G}}\|_{\text{F}}\|\bm{U}_1-\bm{\overline{U}}_1\|_{2,1}\leq\epsilon/4.
	\end{equation*}
	Similarly, we can obtain the upper bound for the third and the last term, and thus show that $\|\cm{T}-\cm{\overline{T}}\|_{\text{F}}\leq\epsilon$.

\end{proof}

\begin{lemma} \label{lemma:RSC} (Restricted strong convexity) Suppose that $T\gtrsim NP$, with probability at least $1-C\exp[-c\log(NP)]$,
	\begin{equation*}
	\alpha_{\textup{RSC}}\|\bm{\Delta}\|_{\textup{F}}^2\leq\frac{1}{T}\sum_{t=1}^T\|\bm{\Delta}_{(1)}\bm{x}_t\|,
	\end{equation*}
	where $\alpha_{\text{RSC}}=\lambda_{\min}(\bm{\Sigma}_{\bm{e}})/(2\mu_{\max}(\mathcal{A}))$.
\end{lemma}

\begin{proof}
	Denote $\bm{X}=(\bm{x}_1,\dots,\bm{x}_{T})$, $\bm{\widehat{\Gamma}}=\bm{XX}^\top/T$, $\bm{\Gamma}=\mathbb{E}\bm{\widehat{\Gamma}}$ and $\bm{\delta}=\text{vec}(\bm{\Delta}_{(1)})$. Note that
	\begin{equation*}
	T^{-1}\sum_{t=1}^T\|\bm{\Delta}_{(1)}\bm{x}_{t}\|_2^2=T^{-1}\|\bm{\Delta}_{(1)}\bm{X}\|_{\text{F}}^2=T^{-1}\|(\bm{X}^\top\otimes\bm{I}_{N})\bm{\delta}\|_2^2=\bm{\delta}^\top[(\bm{X}\bm{X}^\top/T)\otimes\bm{I}_{N}]\bm{\delta}.
	\end{equation*}
	Thus, the objective is to show $\delta^\top[(\bm{X}\bm{X}^\top/T)\otimes\bm{I}_{N}]\delta$ is lower bounded away from zero.

	Since $\bm{\widehat{\Gamma}}=\bm{\Gamma}+(\bm{\widehat{\Gamma}}-\bm{\Gamma})$
	and by spectral measure, we have $\lambda_{\min}(\bm{\Gamma})\geq\lambda_{\min}(\bm{\Sigma}_{\bm{e}})/\mu_{\max}(\mathcal{A})$. Then, it suffices to show that $\bm{\delta}^\top[(\bm{\widehat{\Gamma}}-\bm{\Gamma})\otimes\bm{I}_N]\bm{\delta}$ does not deviate much from zero for any $\bm{\Delta}\in\mathcal{S}(2r_1,2r_2,2r_3)$.

	By Proposition 2.4 in \citet{basu2015regularized}, for any single vector $\bm{v}\in\mathbb{R}^{N^2P-1}$ such that $\|\bm{v}\|_2=1$, any $\eta>0$,
	\begin{equation*}
	\mathbb{P}\left(|\bm{v}^\top[(\bm{\widehat{\Gamma}}-\bm{\Gamma})\otimes\bm{I}_N]\bm{v}|>2\pi\mathcal{M}(f_{\bm{X}})\eta\right)<2\exp[-cT\min(\eta,\eta^2)].
	\end{equation*}
	Then, we can extend this deviation bound to the union bound on the set $\mathcal{S}(2r_1,2r_2,2r_3)$. By Lemma \ref{lemma:covering}, for $\mathcal{S}(2r_1,2r_2,2r_3)$, we can construct a $\epsilon$-net of cardinality at most $(12/\epsilon)^{r_1r_2r_3+Nr_1+Nr_2+Pr_3}$ and approximate the deviation on this net, which yields that for some $\kappa>1$,
	\begin{equation*}\begin{split}
	&\mathbb{P}\left[\sup_{\bm{\Delta}\in\mathcal{S}(2r_1,2r_2,2r_3)}|\bm{\delta}^\top[(\bm{\widehat{\Gamma}}-\bm{\Gamma})\otimes\bm{I}_N]\bm{\delta}|>2\pi\kappa\mathcal{M}(f_{\bm{X}})\eta\right]\\
	\leq&2\exp[-cT\min(\eta,\eta^2)+(r_1r_2r_3+Nr_1+Nr_2+Pr_3)\log(12/\epsilon)].
	\end{split}\end{equation*}
	Then, we can set $\eta=[\lambda_{\min}(\bm{\Sigma}_{\bm{e}})]/[4\pi\kappa\mathcal{M}(f_{\bm{X}})\mu_{\min}(\mathcal{A})]<1$, and then we obtain that for $T\gtrsim (r_1r_2r_3+Nr_1+Nr_2+Pr_3)[(\lambda_{\max}(\bm{\Sigma}_{\bm{e}})\mu_{\max}(\mathcal{A}))/(\lambda_{\min}(\bm{\Sigma}_{\bm{e}})\mu_{\min}(\mathcal{A}))]^2$,
	\begin{equation*}
	\mathbb{P}\left[\sup_{\bm{\Delta}\in\mathcal{S}(2r_1,2r_2,2r_3)}\bm{\delta}^\top(\bm{\widehat{\Gamma}}\otimes\bm{I}_N)\bm{\delta}\leq\frac{\lambda_{\min}(\bm{\Sigma}_{\bm{e}})}{2\mu_{\max}(\mathcal{A})}\right]
	\leq C\exp[-c(r_1r_2r_3+Nr_1+Nr_2+Pr_3)].
	\end{equation*}
	Therefore,
	\begin{equation*}
	\mathbb{P}\left[T^{-1}\sum_{t=1}^T\|\bm{\Delta}_{(1)}\bm{x}_{t}\|_2^2\geq\frac{\lambda_{\min}(\bm{\Sigma}_{\bm{e}})}{2\mu_{\max}(\mathcal{A})}\|\bm{\Delta}\|_{\text{F}}^2\right]\leq C\exp[-c(r_1r_2r_3+Nr_1+Nr_2+Pr_3)].
	\end{equation*}
\end{proof}

The following Lemma is the concentration of $\beta$-mixing subgaussian random variables in \citet{wong2019lasso}.
\begin{lemma}
	\label{lemma:beta_mixing}
	Let $\bm{z}=(Z_1,\dots,Z_T)$ consist of a sequence of mean-zero random variables with exponentially decaying $\beta$-mixing coefficients, i.e. there exists some constant $c_\beta>0$ such that $\forall l\geq1$, $\beta(l)\leq\exp(-c_\beta l)$. Let $K$ be such that $\max_{t=1}^T\|Z_t\|_{\psi_2}\leq\sqrt{K}$. Choose a block length $a_T\geq1$ and let $\mu_T=\lfloor T/(2a_T)\rfloor$. We have, for any $t>0$,
	\begin{equation*}\begin{split}
	\mathbb{P}\left(\frac{1}{T}\left|\|\bm{z}\|_2^2-\mathbb{E}\|\bm{z}\|_2^2\right|>t\right)&\leq4\exp\left(-C{\min}\left\{\frac{t^2\mu_T}{K^2},\frac{t\mu_T}{K}\right\}\right)\\
	&+2(\mu_T-1)\exp(-c_\beta a_T)+\exp\left(-\frac{2t\mu_T}{K}\right).
	\end{split}\end{equation*}
\end{lemma}

\section{Real Dataset Information}

The US macroeconomics dataset is provided by \citet{koop2013forecasting}. The dataset includes a list of 40 quarterly macroeconomic variables of the United States, from Q1-1959 to Q4-2007. All series are transformed to be stationary as in Table 1, standardized to have zero mean and unit variance, and seasonally adjusted except for financial variables. These forty macroeconomic variables capture many aspects of the economy (e.g. production, price, interest rate, consumption, labor, stock markets and exchange rates) and many empirical econometric literature have applied VAR model to these data for structural analysis and forecasting \citep{koop2013forecasting}.

\begin{landscape}
	\begin{table}[]
		\small
		\centering
		\caption{Forty quarterly macroeconomic variables. Except for financial variables, variables are seasonally adjusted. All variables are transformed to stationarity with the following transformation codes. The code represents: 1 = no transformation, 2 = first difference, 3 = second difference, 4 = log, 5 = first difference of logged variables, 6 = second difference of logged variables.}
		\label{my-label}
		\begin{tabular}{@{}llllll@{}}
			\toprule
			Short name & Code & Description                                    & Short name & Code & Description                                     \\ \midrule
			GDP251     & 5    & Real GDP, quantity index (2000=100)            & SEYGT10    & 1    & Spread btwn 10-yr and 3-mth T-bill rates        \\[-2ex]
			CPIAUCSL   & 6    & CPI all items                                  & HHSNTN     & 2    & Univ of Mich index of consumer expectations     \\[-2ex]
			FYFF       & 2    & Interest rate: federal funds (\% per annum)    & PMI        & 1    & Purchasing managers' index                      \\[-2ex]
			PSCCOMR    & 5    & Real spot market price index: all commodities  & PMDEL      & 1    & NAPM vendor deliveries index (\%)               \\[-2ex]
			FMRNBA     & 3    & Depository inst reserves: nonborrowed (mil\$)  & PMCP       & 1    & NAPM commodity price index (\%)                 \\[-2ex]
			FMRRA      & 6    & Depository inst reserves: total (mil\$)        & GDP256     & 5    & Real gross private domestic investment          \\[-2ex]
			FM2        & 6    & Money stock: M2 (bil\$)                        & LBOUT      & 5    & Output per hr: all persons, business sec        \\[-2ex]
			GDP252     & 5    & Real Personal Cons. Exp., Quantity Index       & PMNV       & 1    & NAPM inventories index (\%)                     \\[-2ex]
			IPS10      & 5    & Industrial production index: total             & GDP263     & 5    & Real exports                                    \\[-2ex]
			UTL11      & 1    & Capacity utilization: manufacturing (SIC)      & GDP264     & 5    & Real imports                                    \\[-2ex]
			LHUR       & 2    & Unemp. rate: All workers, 16 and over (\%)     & GDP 265    & 5    & Real govt cons expenditures \& gross investment \\[-2ex]
			HSFR       & 4    & Housing starts: Total (thousands)              & LBMNU      & 5    & Hrs of all persons: nonfarm business sector     \\[-2ex]
			PWFSA      & 6    & Producer price index: finished goods           & PMNO       & 1    & NAPM new orders index (\%)                      \\[-2ex]
			GDP273     & 6    & Personal Consumption Exp.: price index         & CCINRV     & 6    & Consumer credit outstanding: nonrevolving       \\[-2ex]
			CES275R    & 5    & Real avg hrly earnings, non-farm prod. workers & BUSLOANS   & 6    & Comm. and industrial loans at all comm. banks   \\[-2ex]
			FM1        & 6    & Money stock: M1 (bil\$)                        & PMP        & 1    & NAPM production index (\%)                      \\[-2ex]
			FSPIN      & 5    & S\&P's common stock price index: industrials   & GDP276\_1  & 6    & Housing price index                             \\[-2ex]
			FYGT10     & 2    & Interest rate: US treasury const. mat., 10-yr  & GDP270     & 5    & Real final sales to domestic purchasers         \\[-2ex]
			EXRUS      & 5    & US effective exchange rate: index number       & GDP253     & 5    & Real personal cons expenditures: durable goods  \\[-2ex]
			CES002     & 5    & Employees, nonfarm: total private              & LHEL       & 2    & Index of help-wanted ads in newspapers         \\[-1ex] \bottomrule
		\end{tabular}
	\end{table}
\end{landscape}

\bibliography{CAN}

\end{document}